\documentclass[lettersize,journal,colorlinks=true,linkcolor=blue,citecolor=blue,urlcolor=blue]{IEEEtran}
\usepackage{amsmath,amsfonts,amssymb}
\usepackage{amsthm}
\usepackage{orcidlink}
\newtheorem{theorem}{Theorem}

\newtheorem{proposition}{Proposition}
\newtheorem{remark}{Remark}
\usepackage{algorithmic}
\usepackage{algorithm}
\usepackage{array}
\usepackage[caption=false,font=normalsize,labelfont=sf,textfont=sf]{subfig}
\usepackage{textcomp}
\usepackage{stfloats}
\usepackage{url}
\usepackage{graphicx}
\usepackage{cite}
\usepackage{booktabs}
\usepackage{bm}
\usepackage{color}
\usepackage{enumitem}
\usepackage[utf8]{inputenc}
\usepackage{newunicodechar}
\newunicodechar{̣}{}
\usepackage{tikz}
\usetikzlibrary{arrows.meta, positioning, shapes.geometric, fit, patterns, calc, decorations.pathreplacing}

\newcommand{\xhat}{\hat{\bm{x}}}

\newcommand{\ztrue}{\bm{z}^*}
\newcommand{\xslow}{\bm{x}_{\text{s}}}
\newcommand{\xfast}{\bm{x}_{\text{f}}}

\newcommand{\dotxslow}{\dot{\bm{x}}_{\text{s}}}
\newcommand{\xshat}{\hat{\bm{x}}_{\text{s}}}
\newcommand{\R}{\mathbb{R}}

\begin{document}

\title{Physics-Guided Dimension Reduction for Simulation-Free Operator Learning of Stiff Differential--Algebraic Systems}

\author{Huy~Hoang~Le\,\orcidlink{0009-0009-2138-5468},~\IEEEmembership{}
        Haoguang~Wang\,\orcidlink{0009-0005-4101-864X},~\IEEEmembership{}
        Christian~Moya\,\orcidlink{0000-0003-0180-9285},~\IEEEmembership{}
        Marcos~Netto\,\orcidlink{0000-0001-7002-3345},~\IEEEmembership{Senior Member,~IEEE,}
        and~Guang~Lin*\,\orcidlink{0000-0002-0976-1987}~\IEEEmembership{}%
\thanks{Manuscript received XX, 2026.}%
\thanks{This work is supported in part by the National Science Foundation under Grant 2328241. G. Lin would like to thank the support of National Science Foundation (DMS-2533878, DMS-2053746, DMS-2134209, ECCS-2328241, CBET-2347401 and OAC-2311848), and U.S.~Department of Energy (DOE) Office of Science Advanced Scientific Computing Research program DE-SC0023161, the SciDAC LEADS Institute, and DOE–Fusion Energy Science, under grant number: DE-SC0024583.}%
\thanks{H.H.~Le, H.~Wang, and G.~Lin are with the School of Mechanical Engineering,
  Purdue University, West Lafayette, IN~47907~USA
  (e-mail: lehuyhoang02012002@gmail.com; wang6411@purdue.edu; guanglin@purdue.edu).}%
\thanks{C.~Moya and G.~Lin are with the Department of Mathematics,
  Purdue University, West Lafayette, IN~47907~USA
  (e-mail: cmoyacal@purdue.edu).}%
\thanks{M.~Netto is with the Department of Electrical and Computer Engineering,
  New Jersey Institute of Technology, Newark, NJ~07102~USA
  (e-mail: marcos.netto@njit.edu).}%
\thanks{*Corresponding author: G.~Lin.}}

\maketitle

% ══════════════════════════════════════════
\begin{abstract}
% ══════════════════════════════════════════
Neural surrogates for stiff differential--algebraic equations (DAEs) face two barriers: soft-constraint methods leave algebraic residuals that stiffness amplifies into errors, and hard-constraint methods require trajectory data from stiff integrators. We introduce an \textbf{extended Newton implicit layer} that enforces algebraic constraints exactly and reduces fast dynamics to their quasi-steady-state values in a single differentiable solve. Embedded in a physics-informed DeepONet, the layer recovers all fast and algebraic states exactly from slow-state predictions, removes the per-window stiffness-amplification pathway, and yields a stiffness-scaled Implicit Function Theorem gradient absent from penalty methods. Cascaded implicit layers extend this to multi-component systems with provable convergence. On a grid-forming inverter (stiffness ratio $\approx 4{,}712$), extended Newton attains 1.42\% error vs.\ 39.3\% (penalty) and 57.0\% (standard Newton); augmented Lagrangian and feedback linearization diverged. Two independently trained models compose without retraining (0.72\%--1.16\% error, exact constraint satisfaction). Cross-domain validation on the Robertson stiff DAE (stiffness ratio up to $10^{5}$) confirms generalization. Conformal prediction provides 90\% coverage with automatic out-of-distribution detection.
\end{abstract}


\section*{Acronyms}
\begin{description}[labelwidth=0.9cm,leftmargin=1.1cm,nosep,font=\normalfont]
\item[AL] Augmented Lagrangian
\item[CP] Conformal prediction
\item[DAE] Differential-algebraic equation
\item[FL] Feedback linearization
\item[GFM] Grid-forming (inverter)
\item[IFT] Implicit Function Theorem
\item[KKT] Karush--Kuhn--Tucker
\item[MLP] Multi-layer perceptron
\item[ODE] Ordinary differential equation
\item[OOD] Out-of-distribution
\item[PCC] Point of common coupling
\item[PINN] Physics-informed neural network
\item[QSS] Quasi-steady-state
\item[SMIB] Single-machine infinite bus
\item[UQ] Uncertainty quantification
\end{description}

% ══════════════════════════════════════════
% ══════════════════════════════════════════
\section{Introduction}
% ══════════════════════════════════════════

\IEEEPARstart{S}{emi}--explicit index-1 differential-algebraic equations (DAEs),
\begin{align}
    \dot{\bm{x}} &= \bm{f}(\bm{x}, \bm{z}, \bm{\mu}), \quad \bm{x} \in \R^{n_x} \label{eq:dae_ode}\\
    \bm{0} &= \bm{g}(\bm{x}, \bm{z}, \bm{\mu}), \quad \bm{z} \in \R^{n_z} \label{eq:dae_alg}
\end{align}
where $\bm{x} \in \R^{n_x}$, $\bm{z} \in \R^{n_z}$, and $\bm{\mu}$ denote, respectively, states, algebraic variables, and parameters, model systems such as power grids~\cite{stiasny2024pinnsim}, chemical reaction networks~\cite{brenan1996dae}, and multi-body mechanical systems~\cite{hairer1996ode2}. Neural surrogates approximate the solution operator $(\bm{x}_0, \bm{\mu}) \mapsto \bm{x}(t)$ orders of magnitude faster than implicit integrators, but on stiff systems they face a central challenge: satisfying algebraic constraints without trajectory data.

\subsection{The Data--Constraint Dilemma}
\label{sec:dilemma}

Learning DAE surrogates forces a trade-off between data efficiency and constraint satisfaction:

\textbf{Path~A: Soft constraints, no data.} Physics-informed neural networks (PINNs) penalize the algebraic residual $\|\bm{g}\|^2$ during training (e.g., DAE-PINN~\cite{moya2023daepinn}, augmented-Lagrangian PINNs (AL-PINNs)~\cite{son2023alpinn}). Although this avoids simulation data, it empirically leaves $\|\bm{g}\| \sim 10^{-2}$ to $10^{-5}$. In stiff systems, these small residuals amplify into large derivative errors, degrading convergence for fast states.

\textbf{Path~B: Hard constraints, need data.} Newton/Karush--Kuhn--Tucker (KKT) layers enforce $\bm{g}=\bm{0}$ exactly~\cite{donti2021dc3, pal2025pnodes, golder2025daehardnet}, but require trajectory data from stiff integrators, the very bottleneck surrogates aim to bypass.

\textbf{Path~C: Hard constraints, data-free, scalable.} Embedding a Newton solver within a physics-informed operator network combines the strengths of A and B. The challenge is scalability: the monolithic Newton system grows cubically with the number of components. We resolve this with cascaded implicit layers (Section~\ref{sec:cascaded}), recovering linear scaling.

\subsection{The Dimension Reduction Opportunity}
\label{sec:stiff-aware_problem}

Resolving the data--constraint dilemma alone leaves physical structure unexploited. Neural operator surrogates predict trajectories one window at a time, yet not all $n_x$ states are irreducible. For stiff DAEs, the states partition as
\begin{equation}
    \bm{x} = (\xslow, \xfast), \quad \xslow \in \R^{n_\text{s}},\; \xfast \in \R^{n_\text{f}},\; n_\text{s} + n_\text{f} = n_x
    \label{eq:decomp}
\end{equation}
where $\xfast$ operates on timescales $\tau_\text{f} \ll T_w$ and $\xslow$ has $\tau_\text{s} \gg \tau_\text{f}$. Fast states (e.g., converter currents) settle to quasi-steady-state (QSS) within each window, satisfying $\bm{f}_\text{fast} \approx \bm{0}$ at the boundary; this gives Newton additional equations analogous to $\bm{g} = \bm{0}$. Predicting fast states forces the network to resolve $O(1/\kappa)$ oscillations unnecessarily. Instead, the network predicts only the slow states and Newton recovers the rest to machine precision.

\subsection{Contributions}

We propose a unified framework in which a Newton solver resolves all algebraic and quasi-steady-state constraints, leaving the network to learn only the irreducible slow dynamics. To our knowledge, this is the first approach to simultaneously achieve hard constraint enforcement, simulation-free training, and stiffness-aware dimension reduction for DAEs.
\begin{enumerate}[label=(\roman*)]
    \item \textbf{Extended Newton implicit layer with singular perturbation reduction.} A single differentiable layer embedded in a PI-DeepONet~\cite{wang2021pideepont} unifies algebraic constraint enforcement ($\bm{g}=\bm{0}$) and quasi-steady-state reduction ($\bm{f}_\text{fast} \approx \bm{0}$), solving the combined system $[\bm{f}_\text{fast};\, \bm{g}] = \bm{0}$ for all fast and algebraic states from slow-state predictions alone. Three consequences follow: (a)~the network output dimension reduces from $n_x$ to $n_\text{s}$; (b)~the per-window $\kappa$-amplification pathway is removed (Proposition~\ref{prop:rhs_contamination}); and (c)~the Implicit Function Theorem (IFT) gradients capture a $\kappa$-scaled coupling term structurally absent from penalty methods. No trajectory data is required.
    \item \textbf{Cascaded implicit layers for multi-component scalability.} For systems with hierarchical algebraic structure (triangular local Jacobians, sparse network coupling), the extended Newton system decomposes into parallel local solutions plus a small network-level Newton step, with provable linear convergence (Theorem~\ref{thm:convergence}).
\end{enumerate}

Table~\ref{tab:comparison} positions our method among existing DAE surrogate approaches.

\begin{table*}[t]
\centering
\caption{Comparison of DAE surrogate approaches. ``Stiff-aware'' indicates whether the method reduces the network's output dimension by exploiting timescale separation.}
\label{tab:comparison}
\small
\begin{tabular}{@{}lccccccc@{}}
\toprule
 & Exact & Zero & Operator & Stiff & Dim.\ & Scalable & Calibrated \\
 & $g\!=\!0$ & data & learning & DAE & reduction & ($N$) & UQ \\
\midrule
DAE-PINN~\cite{moya2023daepinn} &  & \checkmark &  &  &  &  &  \\
PINNSim~\cite{stiasny2024pinnsim} &  & $\sim$\textsuperscript{a} &  &  &  & \checkmark &  \\
AL-PINN~\cite{son2023alpinn} &  & \checkmark &  &  &  &  &  \\
DC3~\cite{donti2021dc3} & \checkmark &  &  &  &  &  &  \\
PNODEs~\cite{pal2025pnodes} & \checkmark &  &  &  &  &  &  \\
DAE-HardNet~\cite{golder2025daehardnet} & \checkmark &  &  &  &  &  &  \\
PHRPINN~\cite{spotorno2025phrpinn} & \checkmark &  &  &  &  &  &  \\
DeepONet-Grid-UQ~\cite{moya2023deeponetgrid} &  &  & \checkmark &  &  &  & \checkmark$^\dagger$ \\
Koch et al.~\cite{koch2024neural} &  &  & \checkmark &  &  &  &  \\
Feedback lin.\ (FL)~\cite{zhang2025fl} &  & \checkmark &  &  &  &  &  \\
\midrule
Ours (standard Newton) & \checkmark & \checkmark & \checkmark &  &  &  &  \\
\textbf{Ours (extended Newton)} & \checkmark & \checkmark & \checkmark & \checkmark & \checkmark & \checkmark & \checkmark \\
\bottomrule
\multicolumn{8}{l}{\footnotesize $^\dagger$Bayesian UQ (distributional assumptions); ours uses distribution-free conformal prediction.}\\
\multicolumn{8}{l}{\footnotesize \textsuperscript{a}Component PINNs are physics-based (no trajectory data); network coupling uses external Newton--Raphson.}
\end{tabular}
\end{table*}

% ══════════════════════════════════════════
\section{Related Work}
% ══════════════════════════════════════════

\subsection{Soft-Constraint Methods for DAEs}

Soft-constraint approaches penalize the algebraic residual $\|\bm{g}\|^2$ during training. DAE-PINN~\cite{moya2023daepinn} combines implicit Runge--Kutta time-stepping with penalty-weighted residuals; AL-PINNs~\cite{son2023alpinn} replaces the fixed penalty with an augmented Lagrangian scheme. Trust-region sequential quadratic programming (SQP)~\cite{trSQP2024} achieves $100\times$ lower residuals than standard penalties, and feedback linearization (FL)~\cite{zhang2025fl} recasts the loss with a control-inspired multiplier update; both still drive $\|\bm{g}\| \to 0$ only asymptotically. At the system level, PINNSim~\cite{stiasny2024pinnsim} trains component-level PINNs separately and couples them via an external Newton step, but this non-differentiable coupling blocks gradient flow. Koch et al.~\cite{koch2024neural} train a differentiable algebraic surrogate $\hat{\bm{z}} = h_\phi(\bm{x})$ jointly with the ODE network via operator splitting, avoiding penalty tuning but still leaving $\|\bm{g}(\bm{x}, \hat{\bm{z}})\| > 0$. On stiff systems, this residual is amplified into the fast-state dynamics by a factor proportional to $\kappa$.

\subsection{Hard-Constraint Methods for DAEs}

Hard-constraint methods embed a solver in the forward pass to enforce $\bm{g}=\bm{0}$, with gradients via the IFT, an approach proposed by Donti et al.~\cite{donti2021dc3} for static constrained optimization. For dynamical systems, methods include manifold projection~\cite{pal2025pnodes}, KKT projection~\cite{golder2025daehardnet}, predict--project~\cite{spotorno2025phrpinn}, stabilized neural differential equations (Baumgarte stabilization)~\cite{white2023sndes}, and simultaneous collocation via the Ipopt nonlinear programming solver~\cite{lueg2025simultaneous}. All share two limitations: they require trajectory data and treat all constraints monolithically, without exploiting timescale separation.

\subsection{Operator Learning for Dynamical Systems}

PI-DeepONet~\cite{wang2021pideepont} extends DeepONet~\cite{lu2021deeponet} with physics-informed training for ODEs and PDEs. Recent extensions include PI neural operators for power system dynamics~\cite{karampinis2025pioperator} and DeepONet-Grid-UQ~\cite{moya2023deeponetgrid} with Bayesian uncertainty quantification. All address ODE systems; none combines operator learning with hard DAE constraints. Our extended Newton layer bridges this gap by exploiting timescale separation.

\subsection{Singular Perturbation in Numerical Methods and Machine Learning}

Singular perturbation theory~\cite{tikhonov1952, kokotovic1999} decomposes stiff ODEs into slow and fast subsystems via the QSS approximation, $\bm{f}_\text{fast} \approx \bm{0}$, enabling model reduction~\cite{gear2006} and stable discretization~\cite{zhao2025gfm}. Ji et al.~\cite{ji2021stiffpinn} apply QSS as a preprocessing step to reduce stiffness before PINN training, demonstrated on the Robertson kinetics benchmark we revisit in Section~\ref{sec:robertson}. In machine learning, Lee and Temam~\cite{lee2025fsnn} propose fast-slow neural networks that encode timescale separation architecturally from trajectory data. Caldana et al.~\cite{caldana2025neuralode} address stiffness in neural ODE surrogates through data-driven time reparametrization. All three are data-driven and do not enforce algebraic constraints. Our extended Newton layer instead embeds the QSS solution within a differentiable implicit layer, enforcing constraints exactly without manual equation rewriting.

\subsection{Uncertainty Quantification for Neural Surrogates}

Conformal prediction (CP)~\cite{angelopoulos2023conformal, shafer2008tutorial} provides distribution-free prediction sets from held-out residuals. Extensions include conformalized quantile regression~\cite{romano2019conformalized}, adaptive conformal inference~\cite{gibbs2021adaptive}, and multi-step time-series forecasting~\cite{stankeviciute2021conformal}. Existing uncertainty quantification for power-system surrogates is primarily Bayesian~\cite{moya2023deeponetgrid}; no prior work applies CP to hard-constrained DAE surrogates. In our framework, Newton-solved states inherit uncertainty from slow-state predictions, unlike standard time-series CP~\cite{stankeviciute2021conformal}; we use CP as a diagnostic (\S\ref{sec:cp_exp}).

% ══════════════════════════════════════════
\section{Problem Formulation}
% ══════════════════════════════════════════

\subsection{Semi-Explicit DAE}

Consider the semi-explicit index-1 DAE~\eqref{eq:dae_ode}--\eqref{eq:dae_alg}. Dropping the parameter $\bm{\mu}$ for brevity, we assume $\bm{J}_z := \partial \bm{g}/\partial \bm{z}$ is nonsingular (index-1 regularity). The IFT guarantees a unique smooth mapping $\bm{z} = \bm{\phi}(\bm{x})$ satisfying $\bm{g}(\bm{x}, \bm{\phi}(\bm{x})) = \bm{0}$. Evaluating the ODE vector field requires this mapping, available only through Newton iteration or trajectory data. This is the data--constraint dilemma of \S\ref{sec:dilemma}.

\subsection{Stiff Decomposition: Slow and Fast Modes}
\label{sec:stiff_decomp}

For prediction windows $T_w \gg \tau_\text{f}$, the fast states have settled to quasi-steady-state by the end of each window:
\begin{equation}
    \bm{f}_\text{fast}(\xslow, \xfast, \bm{z}) \approx \bm{0}.
    \label{eq:qss}
\end{equation}
This is the standard QSS approximation~\cite{tikhonov1952, kokotovic1999}. Combined with~\eqref{eq:dae_alg}, the fast and algebraic variables satisfy:
\begin{equation}
    \bm{F}(\xslow, \bm{y}) = \begin{pmatrix} \bm{f}_\text{fast}(\xslow, \xfast, \bm{z}) \\ \bm{g}(\xslow, \xfast, \bm{z}) \end{pmatrix} = \bm{0}, \quad \bm{y} = (\xfast, \bm{z}) \in \R^{n_\text{f}+n_z}
    \label{eq:extended_system}
\end{equation}
which Newton solves for $\bm{y}$ given $\xslow$, extending its scope from $n_z$ algebraic unknowns to $n_z + n_\text{f}$ unknowns using only the governing equations.

\begin{proposition}[Extended system regularity]\label{prop:regularity}
If the DAE~\eqref{eq:dae_ode}--\eqref{eq:dae_alg} is index-1 (i.e., $\bm{J}_z$ is nonsingular), and the reduced fast subsystem, obtained by eliminating $\bm{z}$ via the IFT, has all eigenvalues with strictly negative real parts, then the combined Jacobian $\partial \bm{F}/\partial \bm{y}$ is nonsingular and Newton converges locally for~\eqref{eq:extended_system}.
\end{proposition}
\begin{proof}
The combined Jacobian has block structure
$\partial \bm{F}/\partial \bm{y} = \bigl(\begin{smallmatrix} \bm{A} & \bm{B} \\ \bm{C} & \bm{J}_z \end{smallmatrix}\bigr)$,
where $\bm{A} = \partial \bm{f}_\text{fast}/\partial \xfast$, $\bm{B} = \partial \bm{f}_\text{fast}/\partial \bm{z}$, $\bm{C} = \partial \bm{g}/\partial \xfast$, and all derivatives are partial (other variables held fixed). Using $\bm{J}_z$ as the Schur pivot:
\begin{equation}
    \det\!\left(\frac{\partial \bm{F}}{\partial \bm{y}}\right) = \det(\bm{J}_z)\;\det(\bm{S}),
    \label{eq:schur}
\end{equation}

\noindent
where $\bm{S} = \bm{A} - \bm{B}\,\bm{J}_z^{-1}\bm{C}$. The first term in~\eqref{eq:schur} is nonsingular by the index-1 assumption. The Schur complement $\bm{S}$ is precisely the total derivative of $\bm{f}_\text{fast}$ with respect to $\xfast$ after eliminating $\bm{z}$ via $\bm{z}^* = \bm{\phi}(\bm{x})$:
\begin{equation*}
    \bm{S} = \frac{\partial \bm{f}_\text{fast}}{\partial \xfast} - \frac{\partial \bm{f}_\text{fast}}{\partial \bm{z}}\,\bm{J}_z^{-1}\frac{\partial \bm{g}}{\partial \xfast} = \frac{d\bm{f}_\text{fast}}{d\xfast}\bigg|_{\bm{g}=\bm{0}}.
\end{equation*}
This is the Jacobian of the reduced fast dynamics on the constraint manifold, whose eigenvalues coincide with the fast eigenvalues of the original DAE. By hypothesis, all have strictly negative real parts, so $\bm{S}$ is nonsingular. Both terms in~\eqref{eq:schur} are nonsingular, hence $\partial \bm{F}/\partial \bm{y}$ is nonsingular, and Newton converges locally.
\end{proof}

\begin{remark}[Failure modes]\label{rmk:failure}
Regularity of $\partial\bm{F}/\partial\bm{y}$ requires both terms in~\eqref{eq:schur} to be nonsingular. This can be violated when: (a)~fast modes become weakly damped (decay rate $\alpha \to 0$), e.g., near oscillatory bifurcations such as a Hopf bifurcation; or (b)~the algebraic Jacobian $\bm{J}_z$ approaches singularity. In practice, the Newton residual $\|\bm{F}\|$ provides an online diagnostic: if Newton fails to converge within a prescribed iteration budget, the surrogate flags unreliable predictions.
\end{remark}

\begin{proposition}[QSS approximation error]\label{prop:qss_error}
Let $\xfast^\text{true}(t)$ be the true fast state at the 
end of a prediction window $t = T_w$, and let $\xfast^*$ 
be the QSS solution from~\eqref{eq:extended_system} 
evaluated at $\xslow(T_w)$. Assume $\xslow(t) \in C^1[0, T_w]$ with bounded $\|\dotxslow\|_\infty$, and that the reduced fast Jacobian $\bm{S}$ (defined in Proposition~\ref{prop:regularity}) is uniformly Hurwitz along the trajectory with minimum decay rate $\alpha = \min_i |\mathrm{Re}(\lambda_i(\bm{S}))| > 0$. Then:
\begin{equation}
    \|\xfast^\text{true}(T_w) - \xfast^*\| \leq C_1 e^{-\alpha T_w} + C_2 \frac{\|\dot{\xslow}\|_\infty}{\alpha}
    \label{eq:qss_bound}
\end{equation}
where $C_1$ depends on the initial fast-state deviation $\|\xfast(0) - \xfast^*(0)\|$. The constant $C_2 = \sup_{t \in [0, T_w]} 
\|\bm{S}^{-1}
(\partial \bm{f}_\text{fast}/\partial \xslow)\|$ 
bounds the sensitivity of $\xfast^*$ to slow-state variation along the trajectory, and remains bounded as long as $\bm{S}$ stays uniformly Hurwitz.
\end{proposition}
\begin{proof}[Proof sketch]
By Tikhonov's theorem~\cite{tikhonov1952}, the fast states track $\xfast^*(\xslow(t))$ with an exponentially decaying boundary layer of $O(e^{-\alpha t})$. The residual tracking error arises because $\xslow$ is not constant: as $\xslow$ drifts by $O(\|\dotxslow\|_\infty \cdot \tau)$ over timescale $\tau$, the fast states lag by $O(C_2 \|\dotxslow\|_\infty/\alpha)$. For $T_w \gg 1/\alpha$, the first term vanishes and the QSS error is $O(\|\dotxslow\|_\infty/\alpha)$.
\end{proof}

For systems with $\alpha \gg 1/T_w$, the first term vanishes exponentially and the QSS error is dominated by the tracking residual $O(C_2\|\dotxslow\|_\infty/\alpha)$.

\subsection{Multi-Component Extension}
\label{sec:multi_component}

Engineered systems often comprise $N$ interconnected components:
\begin{align}
    \dot{\bm{x}}_i &= \bm{f}_i(\bm{x}_i, \bm{z}_i, \bm{v}), & i &= 1,\ldots,N \label{eq:local_ode}\\
    \bm{0} &= \bm{g}_i^\text{local}(\bm{x}_i, \bm{z}_i), & i &= 1,\ldots,N \label{eq:local_alg}
\end{align}
coupled through shared network variables $\bm{v} \in \R^{n_v}$:
\begin{equation}
    \bm{0} = \bm{g}^\text{net}(\bm{z}_1, \ldots, \bm{z}_N, \bm{v}). \label{eq:net_alg}
\end{equation}
If the local Jacobians $\bm{J}_{z_i} = \partial\bm{g}_i^\text{local}/\partial \bm{z}_i$ are triangular (common in controller cascades), local constraints are solved by forward substitution at $O(n_{z_i}^2)$ per component; only the network coupling~\eqref{eq:net_alg} requires iteration (\S\ref{sec:cascaded}).

% ══════════════════════════════════════════
\section{Method}
% ══════════════════════════════════════════

Our framework has three components: (\S\ref{sec:newton_deeponet}) a Newton-in-the-loop PI-DeepONet for simulation-free training with exact constraints, (\S\ref{sec:extended_newton}) an extended Newton layer that reduces the output dimension via timescale separation, and (\S\ref{sec:cascaded}) cascaded implicit layers for multi-component scalability. Conformal prediction serves as a diagnostic (\S\ref{sec:cp_exp}).

\subsection{Newton-in-the-Loop PI-DeepONet}
\label{sec:newton_deeponet}

As illustrated in Fig.~\ref{fig:architecture}, the PI-DeepONet predicts the states from initial conditions and parameters, while a coupled Newton layer computes the corresponding fast states and algebraic variables. This Newton layer resolves the data--constraint dilemma along three axes: (a)~it drives $\|\bm{g}\|$ to machine precision ($\leq 10^{-14}$) regardless of network training quality; (b)~given any $\xhat$, it computes consistent $\ztrue$ without external data, so the ODE residual $\|\dot{\xhat} - \bm{f}(\xhat, \ztrue)\|^2$ is evaluated from the governing equations alone; and (c)~the IFT yields $\partial \ztrue / \partial \xhat = -\bm{J}_z^{-1} (\partial \bm{g}/\partial \bm{x})$, coupling the constraint into the training gradient.

\begin{figure}[t]
    \centering
    \resizebox{\columnwidth}{!}{%
    \begin{tikzpicture}[
        block/.style={draw, rounded corners=2pt, minimum height=0.7cm, minimum width=1.5cm, font=\small, fill=white},
        arr/.style={-{Stealth[length=2mm]}, thick},
        dasharr/.style={-{Stealth[length=2mm]}, thick, dashed},
        every node/.style={font=\small}
    ]
    % ======= ROW 1 (top): Loss + vector field =======
    \node[block, fill=red!10] at (0, 2.6) (loss) {$\mathcal{L}$};
    \node[font=\footnotesize] at (0, 3.15) {$d\xshat/dt = \bm{f}_\text{slow}$};
    \node[block, fill=purple!10, minimum width=1.6cm] at (3.0, 2.6) (rhs) {$\bm{f}_\text{slow}$};
    \draw[arr] (rhs) -- (loss);
    % --- Inputs ---
    \node[font=\footnotesize, align=right] at (-1.8, 0.45) (inp1) {$\bm{x}_{\text{s},0}, \bm{\mu}$\\$\bm{u}(t)$};
    \node[font=\footnotesize] at (-1.8, -0.45) (inp2) {$t$};
    % --- Branch / Trunk ---
    \node[block, fill=blue!8] at (0, 0.45) (branch) {Branch $\bm{b}$};
    \node[block, fill=blue!8] at (0, -0.45) (trunk) {Trunk $\bm{t}$};
    \draw[arr] (inp1) -- (branch); \draw[arr] (inp2) -- (trunk);
    % --- Dot product ---
    \node[draw, circle, inner sep=1.5pt, fill=yellow!15] at (1.6, 0) (dot) {$\otimes$};
    \draw[arr] (branch.east) -- ++(0.2,0) |- (dot);
    \draw[arr] (trunk.east) -- ++(0.2,0) |- (dot);
    % --- xhat ---
    \node[block, fill=green!10] at (3.0, 0) (xhat) {$\xshat$};
    \draw[arr] (dot) -- (xhat);
    % --- Extended Newton box ---
    \draw[dashed, rounded corners=4pt, fill=orange!5, thick]
        (4.0, -1.0) rectangle (7.2, 1.0);
    \node[font=\footnotesize\bfseries] at (5.6, 0.75) {Extended Newton Layer};
    \node[block, fill=orange!15, minimum width=2.4cm] at (5.6, -0.15) (newton) {%
        \begin{tabular}{c}\footnotesize $[\bm{f}_\text{fast};\; \bm{g}] = \bm{0}$\\[-2pt]\footnotesize $\to (\xfast^*, \ztrue)$\end{tabular}};
    \draw[arr] (xhat.east) -- (4.0, 0);
    % --- Output ---
    \node[block, fill=green!10] at (8.6, 0) (yout) {$\xfast^*, \ztrue$};
    \draw[arr] (7.2, 0) -- (yout);
    \draw[arr] (xhat.north) -- ++(0, 0.7) -- (3.0, 2.25);
    \draw[arr] (yout.north) -- (8.6, 2.6) -- (rhs.east);
    \draw[dasharr, blue!60] (loss.south) -- ++(0, -0.35) -- (branch.north)
        node[pos=0.3, left, font=\tiny, blue!60] {$\nabla_\theta$ (IFT)};
    \draw[decorate, decoration={brace, amplitude=4pt, mirror}, thick, gray]
        (-0.85, -1.05) -- (0.85, -1.05)
        node[midway, below=4pt, font=\footnotesize, gray] {PI-DeepONet};
    \end{tikzpicture}%
    }
    \caption{Proposed architecture. The PI-DeepONet predicts slow states $\xshat \in \R^{n_\text{s}}$ from initial conditions $\bm{x}_{\text{s},0}$, parameters $\bm{\mu}$, and external input $\bm{u}(t)$. The extended Newton layer solves $[\bm{f}_\text{fast};\, \bm{g}]=\bm{0}$ for all $n_\text{f}+n_z$ fast states and algebraic variables. IFT gradients (dashed blue) flow back to the network without simulation data.}
    \label{fig:architecture}
\end{figure}
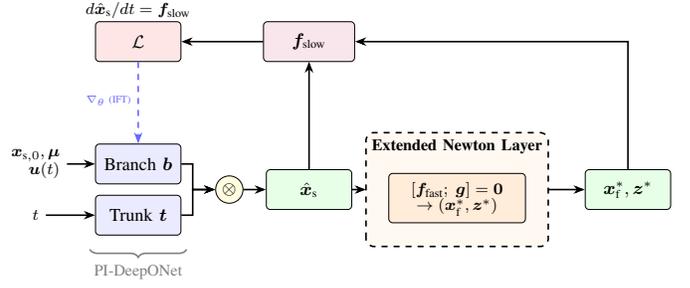

\textbf{Constraint-aware gradient.} The total derivative of $\bm{f}$ w.r.t.\ the predicted states is:
\begin{equation}
    \frac{d\bm{f}}{d\xhat} = \underbrace{\frac{\partial \bm{f}}{\partial \bm{x}}}_{\text{direct}} + \underbrace{\frac{\partial \bm{f}}{\partial \bm{z}} \cdot \frac{\partial \ztrue}{\partial \xhat}}_{\substack{\text{indirect: how } \xhat \text{ errors}\\\text{propagate through } \bm{g}=\bm{0}}}
    \label{eq:total_deriv}
\end{equation}
Penalty methods treat $\|\bm{g}\|^2$ and $\|\dot{\xhat} - \bm{f}\|^2$ as independent losses, capturing only the direct term. The Newton layer provides the indirect coupling automatically via the IFT, which carries $\kappa$-scaled entries that dominate the gradient on stiff systems; extended Newton (\S\ref{sec:extended_newton}) captures both algebraic and QSS coupling through~\eqref{eq:extended_ift}.

The PI-DeepONet uses a branch-trunk architecture~\cite{lu2021deeponet} with Fourier features and residual multi-layer perceptrons (MLPs)~\cite{wang2021pideepont}. Training minimizes the ODE residual with causal weighting:
\begin{equation}
    \mathcal{L} = \sum_{s=1}^{n_\text{out}} \lambda_s \sum_{j=1}^{n_\text{chunks}} w_j \left\| \frac{\partial \xhat_s}{\partial t} - f_s(\xhat, \ztrue) \right\|^2_{\mathcal{T}_j}
    \label{eq:loss}
\end{equation}
where $n_\text{out} = n_\text{s}$ (extended Newton) or $n_x$ (standard Newton), $n_\text{chunks}$ is the number of temporal segments per window, $\lambda_s$ are adaptive per-state weights, and $w_j = \exp(-\epsilon \sum_{k<j} L_k)$ are causal weights~\cite{wang2021pideepont} that down-weight later chunks until earlier ones converge. No $\|\bm{g}\|$ penalty is needed since constraints are enforced architecturally. For $T \gg T_w$, the surrogate applies recursively, with each window's output seeding the next.

\subsection{Extended Newton: Singular Perturbation Reduction}
\label{sec:extended_newton}

The approach of \S\ref{sec:newton_deeponet} still predicts all $n_x$ states; since physics determines the fast states exactly, the network should not predict them. We extend Newton iteration from purely algebraic constraints to the combined system~\eqref{eq:extended_system}, so the PI-DeepONet predicts only the $n_\text{s}$ slow states $\xshat$ and the extended Newton layer solves
\begin{equation}
    \bm{F}(\xshat, \bm{y}) = \begin{pmatrix} \bm{f}_\text{fast}(\xshat, \xfast, \bm{z}) \\ \bm{g}(\xshat, \xfast, \bm{z}) \end{pmatrix} = \bm{0}
    \label{eq:extended_newton}
\end{equation}
for $\bm{y} = (\xfast, \bm{z}) \in \R^{n_\text{f}+n_z}$, using the analytical Jacobian $\partial \bm{F}/\partial \bm{y}$, whose nonsingularity is guaranteed by Proposition~\ref{prop:regularity}. This forms a single system of dimension $n_\text{f}+n_z$: the network predicts $n_\text{s}$ smooth outputs while Newton provides all remaining states exactly ($\|\bm{F}\| \leq 10^{-16}$).

\begin{proposition}[Per-window vector field contamination]\label{prop:rhs_contamination}
Let $\xshat$ be the predicted slow states with error $\epsilon_\text{s} = \|\xshat - \bm{x}_\text{s}^\text{true}\|$, and let $\hat{\bm{x}}_\text{f}$ be the predicted fast states with error $\epsilon_\text{f} = \|\hat{\bm{x}}_\text{f} - \bm{x}_\text{f}^\text{true}\|$. Assume $\bm{f}_\text{slow}$ is locally Lipschitz in $(\bm{x}_\text{s}, \bm{x}_\text{f}, \bm{z})$, and that $\bm{J}_z$ and $\partial\bm{F}/\partial\bm{y}$ remain nonsingular in a neighborhood of the true solution (Proposition~\ref{prop:regularity}). Then within a single prediction window:

(a) Standard Newton (network predicts $\xslow$ and $\xfast$, Newton solves $\bm{g}=\bm{0}$ only):
\begin{equation}
    \|\bm{f}_\text{slow}(\hat{\bm{x}}) - \bm{f}_\text{slow}(\bm{x}^\text{true})\| \leq L_\text{s}\, \epsilon_\text{s} + \underbrace{L_\kappa\, \epsilon_\text{f}}_{\substack{\text{fast-state error}\\\text{amplified by }L_\kappa}}
    \label{eq:rhs_standard}
\end{equation}

(b) Extended Newton (network predicts $\xslow$ only, Newton solves $[\bm{f}_\text{fast};\bm{g}]=\bm{0}$):
\begin{equation}
    \|\bm{f}_\text{slow}(\xshat, \bm{y}^*) - \bm{f}_\text{slow}(\bm{x}_\text{s}^\text{true}, \bm{y}^\text{true})\| \leq L_\text{s}\, \epsilon_\text{s} + \underbrace{C_\text{qss}}_{\substack{\text{QSS error}\\\text{independent of }L_\kappa}}
    \label{eq:rhs_extended}
\end{equation}
where $L_\text{s}$ is the Lipschitz constant of $\bm{f}_\text{slow}$ w.r.t.\ $\bm{x}_\text{s}$, and $L_\kappa$ is the effective fast-to-slow coupling:
\begin{equation}
    L_\kappa = \left\|\frac{\partial \bm{f}_\text{slow}}{\partial \bm{x}_\text{f}}\right\| + \left\|\frac{\partial \bm{f}_\text{slow}}{\partial \bm{z}}\right\| \left\|\bm{J}_z^{-1} \frac{\partial \bm{g}}{\partial \bm{x}_\text{f}}\right\|
    \label{eq:Lkappa_def}
\end{equation}
capturing both direct and indirect (through $\bm{g}=\bm{0}$) coupling. $L_\kappa$ measures how fast-state errors contaminate $\bm{f}_\text{slow}$ (distinct from the timescale ratio $\kappa$). The constant $C_\text{qss} = O(\|\dotxslow\|_\infty/\alpha)$ is the QSS error from Proposition~\ref{prop:qss_error}.
\end{proposition}
\begin{proof}
For (a), Lipschitz expansion of $\bm{f}_\text{slow}$ gives $\|\Delta \bm{f}_\text{slow}\| \leq \|\partial \bm{f}_\text{slow}/\partial \bm{x}_\text{s}\| \cdot \epsilon_\text{s} + \|\partial \bm{f}_\text{slow}/\partial \bm{x}_\text{f}\| \cdot \epsilon_\text{f} + \|\partial \bm{f}_\text{slow}/\partial \bm{z}\| \cdot \|\Delta \bm{z}\|$. Since Newton enforces $\bm{g}=\bm{0}$, the IFT gives $\|\Delta \bm{z}\| = O(\epsilon_\text{s} + \epsilon_\text{f})$. The $L_\kappa$-amplification arises through two paths: the direct coupling $\|\partial \bm{f}_\text{slow}/\partial \bm{x}_\text{f}\|$, and the indirect path $\bm{x}_\text{f} \to \bm{z} \to \bm{f}_\text{slow}$, where $\|\partial \bm{f}_\text{slow}/\partial \bm{z}\| \cdot \|\partial \bm{z}^*/\partial \bm{x}_\text{f}\|$ inherits the algebraic coupling. The combined effect yields~\eqref{eq:rhs_standard}.
For (b), extended Newton solves $\bm{F}(\xshat, \bm{y}^*) = \bm{0}$, determining $\bm{y}^*$ as a function of $\xshat$ alone. By the IFT, $\partial \bm{y}^*/\partial \bm{x}_\text{s} = -(\partial \bm{F}/\partial \bm{y})^{-1}(\partial \bm{F}/\partial \bm{x}_\text{s})$. The key structural property is that $\bm{y}^*$ depends only on $\xshat$, not on any separate prediction $\hat{\bm{x}}_\text{f}$. Hence $\epsilon_\text{f}$ does not appear in the error bound: $\|\bm{y}^* - \bm{y}^{*,\text{true}}\| \leq \|(\partial \bm{F}/\partial \bm{y})^{-1}(\partial \bm{F}/\partial \bm{x}_\text{s})\| \cdot \epsilon_\text{s} + C_\text{qss}$, where the IFT sensitivity $\|(\partial \bm{F}/\partial \bm{y})^{-1}(\partial \bm{F}/\partial \bm{x}_\text{s})\|$ is bounded under the regularity assumptions but may itself be large. The $L_\kappa$-amplification of $\epsilon_\text{f}$ is absent because no fast-state prediction error enters, yielding~\eqref{eq:rhs_extended}.
\end{proof}

The per-window $L_\kappa$-amplification channel of~\eqref{eq:rhs_standard} is thus eliminated; multi-window error accumulation remains possible, validated in \S\ref{sec:gfm_setup}.

\subsubsection{Implicit Gradients for the Extended System}

Applying the IFT to~\eqref{eq:extended_newton}:
\begin{equation}
    \frac{\partial \bm{y}^*}{\partial \xshat} = -\left(\frac{\partial \bm{F}}{\partial \bm{y}}\right)^{-1} \frac{\partial \bm{F}}{\partial \xshat}
    \label{eq:extended_ift}
\end{equation}
where $\partial \bm{F}/\partial \bm{y}$ is nonsingular by Proposition~\ref{prop:regularity}. Both Jacobians can be computed analytically. The loss~\eqref{eq:loss} covers only $n_\text{s}$ slow states, with $\bm{f}_\text{slow}(\xshat, \bm{y}^*)$ evaluated using exact fast and algebraic variables from the solver.

\subsubsection{Generality}

The decomposition~\eqref{eq:decomp} applies to any stiff DAE with timescale separation, requiring only eigenvalue-based slow/fast identification; the extended Newton layer adds no parameters, loss terms, or hyperparameters beyond the partition itself.

\subsection{Cascaded Implicit Layers}
\label{sec:cascaded}

For $N$ coupled components, a monolithic system is $(N(n_\text{f}+n_z)+n_v)$-dimensional and dense (Fig.~\ref{fig:structure_comparison}).

\begin{figure}[t]
    \centering
    \resizebox{\columnwidth}{!}{%
    \begin{tikzpicture}[font=\small]
    % ===== (a) Monolithic =====
    \node[font=\normalsize\bfseries] at (1.4, 3.65) {(a) Monolithic};

    \fill[blue!15] (0,0) rectangle (2.8,2.8);
    \draw[thick] (0,0) rectangle (2.8,2.8);
    \foreach \i in {1,...,6} {
        \draw[gray!40, thin] ({0.4*\i}, 0) -- ({0.4*\i}, 2.8);
        \draw[gray!40, thin] (0, {0.4*\i}) -- (2.8, {0.4*\i});
    }
    \node[font=\scriptsize, gray] at (1.4, 1.4) {dense};

    \node[font=\scriptsize, align=center] at (1.4, -0.6) {$O((N(n_\text{f}{+}n_z){+}n_v)^3)$\\cubic in $N$};

    \node[font=\normalsize, gray] at (3.5, 1.4) {vs.};

    % ===== (b) Cascaded =====
    \node[font=\normalsize\bfseries] at (5.6, 3.65) {(b) Cascaded};

    \draw[thick] (4.2,0) rectangle (7.0,2.8);

    % Block 1
    \fill[orange!25] (4.2,2.3) -- (4.2,2.8) -- (4.7,2.3) -- cycle;
    \draw[orange!70,thick] (4.2,2.8) -- (4.7,2.3);
    \draw[thick] (4.2,2.3) rectangle (4.7,2.8);
    % Block 2
    \fill[orange!25] (4.7,1.8) -- (4.7,2.3) -- (5.2,1.8) -- cycle;
    \draw[orange!70,thick] (4.7,2.3) -- (5.2,1.8);
    \draw[thick] (4.7,1.8) rectangle (5.2,2.3);
    % Block 3
    \fill[orange!25] (5.2,1.3) -- (5.2,1.8) -- (5.7,1.3) -- cycle;
    \draw[orange!70,thick] (5.2,1.8) -- (5.7,1.3);
    \draw[thick] (5.2,1.3) rectangle (5.7,1.8);
    % Block 4
    \fill[orange!25] (5.7,0.8) -- (5.7,1.3) -- (6.2,0.8) -- cycle;
    \draw[orange!70,thick] (5.7,1.3) -- (6.2,0.8);
    \draw[thick] (5.7,0.8) rectangle (6.2,1.3);

    \node[font=\scriptsize, orange!80!black] at (5.2, 0.5) {$\times N$};

    % Network block
    \fill[red!20] (6.2, 0) rectangle (7.0, 0.8);
    \draw[thick] (6.2, 0) rectangle (7.0, 0.8);
    \node[font=\scriptsize, red!70!black, align=center] at (6.6, 0.4) {net\\$n_v$};

    % Zero regions
    \fill[white] (4.2, 0) rectangle (6.2, 0.8);
    \draw[gray!30, thick] (4.2, 0) rectangle (6.2, 0.8);
    \node[font=\scriptsize, gray!50] at (5.2, 0.4) {$\bm{0}$};
    \fill[white] (6.2, 0.8) rectangle (7.0, 2.8);
    \draw[gray!30, thick] (6.2, 0.8) rectangle (7.0, 2.8);
    \node[font=\scriptsize, gray!50] at (6.6, 1.8) {$\bm{0}$};

    % Braces
    \draw[decorate, decoration={brace, amplitude=3pt}, thick, orange!80!black]
        (4.2, 2.9) -- (6.2, 2.9)
        node[midway, above=3pt, font=\scriptsize, orange!80!black] {local ($\parallel$)};
    \draw[decorate, decoration={brace, amplitude=3pt}, thick, red!70!black]
        (6.2, 2.9) -- (7.0, 2.9)
        node[midway, above=3pt, font=\scriptsize, red!70!black] {net};

    \node[font=\scriptsize, align=center] at (5.6, -0.6) {$O(N(n_\text{f}{+}n_z)^2 {+} n_v^3)$\\linear in $N$};

    \end{tikzpicture}%
    }
    \caption{Jacobian structure. \textbf{(a)}~Monolithic: dense system. \textbf{(b)}~Cascaded: $N$ independent local blocks (parallel) plus a small network block.}
    \label{fig:structure_comparison}
\end{figure}

The computation splits into two levels (Algorithm~\ref{alg:cascaded}): each component solves its local $[\bm{f}_{\text{fast},i};\, \bm{g}_i^\text{local}] = \bm{0}$ in parallel with $\bm{v}$ fixed, then a small outer Newton step updates $\bm{v}$ to enforce~\eqref{eq:net_alg}.

\begin{algorithm}[t]
\caption{Cascaded Extended Newton: Forward Pass}
\label{alg:cascaded}
\begin{algorithmic}[1]
\REQUIRE $\{\hat{\bm{x}}_{\text{s},i}\}_{i=1}^N$, initial $\bm{v}^{(0)}$, tolerance $\epsilon$
\FOR{$k = 0, 1, \ldots$ until $\|\bm{g}^\text{net}\| < \epsilon$}
    \FOR{$i = 1, \ldots, N$ \textbf{in parallel}}
        \STATE $\bm{y}_i^{(k)} \leftarrow \text{ExtendedNewton}_i(\hat{\bm{x}}_{\text{s},i}, \bm{v}^{(k)})$
        \COMMENT{Local $[\bm{f}_\text{fast}; \bm{g}^\text{local}] = \bm{0}$}
    \ENDFOR
    \STATE $\bm{r} \leftarrow \bm{g}^\text{net}(\bm{y}_1^{(k)}, \ldots, \bm{y}_N^{(k)}, \bm{v}^{(k)})$
    \STATE $\bm{v}^{(k+1)} \leftarrow \bm{v}^{(k)} - (\partial \bm{g}^\text{net}/\partial \bm{v})^{-1}\bm{r}$
\ENDFOR
\RETURN $\{\bm{y}_i\}_{i=1}^N$, $\bm{v}^*$
\end{algorithmic}
\end{algorithm}

\begin{theorem}[Cascaded convergence]\label{thm:convergence}
Let $\bm{J}_v = \partial \bm{g}^\text{net}/\partial \bm{v}$ be the direct sensitivity of the network residual to $\bm{v}$, and let $\bm{A} = (\partial \bm{g}^\text{net}/\partial \bm{y})(\partial \bm{y}^*/\partial \bm{v})$ be the indirect sensitivity through the local solutions. Define the contraction ratio $\rho = \|\bm{J}_v^{-1} \bm{A}\|$. If $\rho < 1$, then Algorithm~\ref{alg:cascaded} converges linearly: $\|\bm{v}^{(k+1)} - \bm{v}^*\| \leq \rho \|\bm{v}^{(k)} - \bm{v}^*\|$.
\end{theorem}

\begin{proof}
Let $\bm{e}^{(k)} = \bm{v}^{(k)} - \bm{v}^*$. Linearizing around the solution gives
\begin{equation}
    \bm{g}^\text{net}(\bm{y}^*(\bm{v}^{(k)}), \bm{v}^{(k)}) = (\bm{A} + \bm{J}_v)\bm{e}^{(k)} + O(\|\bm{e}^{(k)}\|^2).
    \label{eq:cascaded_linearize}
\end{equation}
The Newton step $\bm{v}^{(k+1)} = \bm{v}^{(k)} - \bm{J}_v^{-1}\bm{g}^\text{net}$ cancels the $\bm{J}_v$ term but not $\bm{A}$, leaving $\bm{e}^{(k+1)} = -\bm{J}_v^{-1}\bm{A}\,\bm{e}^{(k)} + O(\|\bm{e}^{(k)}\|^2)$. Taking norms gives the result.
\end{proof}

For weakly coupled systems ($\rho \ll 1$), a few (2 to 3) outer iterations suffice. Backpropagation via the IFT at both levels yields the exact monolithic gradient.

\begin{remark}[Estimating $\rho$ and safeguards for $\rho \geq 1$]\label{rmk:rho}
The ratio $\rho$ reflects inter-component coupling strength; weakly coupled systems have $\rho \ll 1$. When $\rho \geq 1$ (strong coupling), Algorithm~\ref{alg:cascaded} may diverge. Practical safeguards include: (a)~damped updates $\bm{v}^{(k+1)} = \bm{v}^{(k)} - \eta\,\bm{J}_v^{-1}\bm{r}$ with $\eta < 1$; (b)~Anderson acceleration; or (c)~falling back to a monolithic solve. A sudden increase in outer iterations signals stronger coupling.
\end{remark}

% ══════════════════════════════════════════
\section{Experiments}
% ══════════════════════════════════════════

We validate on stiff DAEs of increasing complexity, from power systems to chemical kinetics. The central experiment (Exp~2) compares extended Newton ($n_\text{s} = 7$) against four baselines on a stiff 21-state inverter system ($\kappa \approx 4{,}712$); cross-domain generalization is tested on the Robertson stiff chemical kinetics problem (Exp~6).

\subsection{Test Systems}
\label{sec:test_systems}

\textbf{System A: Single-machine infinite bus (SMIB).} $n_x=2$, $n_z=1$, $\kappa \sim 10$ (mild stiffness). Both states are slow, so only the base Newton layer is needed.

\textbf{System B: Grid-forming inverter (GFM).} $n_x=13$, $n_z=8$, $\kappa = \omega_b/l_f \approx 4{,}712$~\cite{zhao2025gfm}. The 13 states partition into 7 slow ($\theta$, $p_\text{oc}$, $q_\text{oc}$, $\xi_d$, $\xi_q$, $\gamma_d$, $\gamma_q$; timescales 5--96\,ms) and 6 fast ($i_\text{cv,d/q}$, $v_\text{f,d/q}$, $i_\text{f,d/q}$; timescales 0.4--0.7\,ms). Eigenvalue analysis confirms a spectral gap $|\mathrm{Re}(\lambda_\text{fast})|/|\mathrm{Re}(\lambda_\text{slow})| \geq 7.5$.

The GFM dynamics comprise a synchronization subsystem (frequency droop with low-pass filter) and a $dq$-frame voltage control subsystem (voltage droop, virtual impedance, PI inner-loop controllers, and LCL filter); see~\cite{zhao2025gfm} for the complete state-space derivation. The inner-loop PI integrators ($\xi_{d,q}$, $\gamma_{d,q}$) together with the LCL filter currents and voltages form the $n_\text{f} = 6$ fast states. The $n_z = 8$ algebraic variables arise from coordinate transforms and controller outputs; $\bm{J}_z$ is unit lower-triangular due to the cascading controller structure.

At the nominal operating point, the reduced fast Jacobian $\bm{S}$ (Proposition~\ref{prop:regularity}) has eigenvalues $-1{,}357 \pm 512j$, $-2{,}251 \pm 6{,}200j$, $-2{,}410 \pm 6{,}443j$, with $\sigma_\text{min}(\bm{S}) = 949$ and $\sigma_\text{min}(\partial\bm{F}/\partial\bm{y}) = 0.42$, confirming nonsingularity with large margin. The minimum decay rate $\alpha \approx 1{,}357$\,s$^{-1}$ (Proposition~\ref{prop:qss_error}) gives $e^{-\alpha T_w} \approx 10^{-3}$ for $T_w = 5$\,ms; $C_2 \approx 248$, giving $C_2/\alpha \approx 0.18$; and $L_\kappa \approx 841$ (Proposition~\ref{prop:rhs_contamination}).

\textbf{Baselines.} All share the same 256-hidden, 5-layer modified MLP and predict all 13 states: (a)~Penalty~\cite{moya2023daepinn}: ramped $\|\bm{g}\|^2$ weight; (b)~Augmented Lagrangian~\cite{son2023alpinn}: dual variable updates; (c)~FL~\cite{zhang2025fl}: feedback linearization optimizer; (d)~Standard Newton (ours): Newton solves 8 algebraic states with IFT backward. 
Training budget: 200k epochs (500k for AL). Newton solves use float64 to prevent stiffness-amplified rounding ($\kappa \times 10^{-7} \approx 10^{-4}$); the network remains float32.

\textbf{System C: Robertson stiff chemical kinetics (cross-domain).} $n_x=2$, $n_z=1$, with $\kappa$ ranging $10^2$--$10^5$ across the trajectory~\cite{robertson1966, hairer1996ode2}. The single differential slow state is $y_1$, the fast intermediate is $y_2 \sim 10^{-5}$, and mass conservation $y_1+y_2+y_3=1$ provides the index-1 algebraic state $y_3$. The trajectory spans seven decades, $t \in [10^{-2}, 10^5]$\,s, requiring log-time reparametrization $\tau = \log_{10}(t)$ ($T_w = 0.5$ decade, 14 windows). Used in Exp~6 to test generalization.

\subsection{Exp 1: SMIB, Base Newton Validation}

The SMIB validates the base Newton layer on a scalar algebraic constraint (quadratic in $V$, solved analytically). Three methods share the same architecture (64-basis, 96-hidden, 3-layer MLP, Fourier trunk, 100k epochs):
\begin{itemize}[nosep]
    \item \textbf{Penalty:} predicts all 3 states ($\delta, \omega, V$), penalizes $\|\bm{g}\|^2$ in the loss.
    \item \textbf{FL~\cite{zhang2025fl}:} same output, feedback-linearization optimizer.
    \item \textbf{Newton (ours):} predicts 2 states only; Newton solves $\bm{g}=\bm{0}$ exactly.
\end{itemize}

\begin{figure*}[t]
    \centering
    \includegraphics[width=\textwidth]{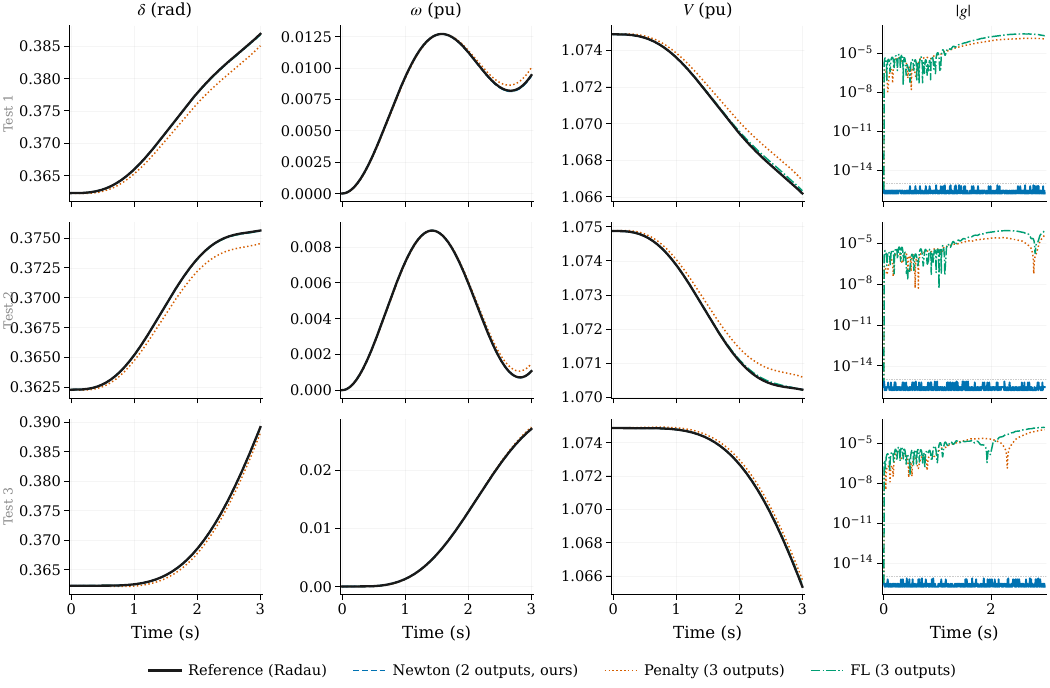}
    \caption{SMIB results on 3 test scenarios (rows). Columns 1--3: trajectories $\delta(t)$, $\omega(t)$, $V(t)$. Column 4: constraint violation $\|\bm{g}\|$ (log scale). Newton (dashed blue) achieves $\|\bm{g}\| \leq 10^{-15}$ by construction; Penalty (dotted) and FL (dash-dot) leave residuals of $10^{-3}$ to $10^{-5}$.}
    \label{fig:smib_combined}
\end{figure*}

Fig.~\ref{fig:smib_combined} and Table~\ref{tab:smib} summarize the results. Newton achieves $\|\bm{g}\| = 2.6 \times 10^{-14}$ and $74\times$ lower $\delta$ error than Penalty, which leaves $\|\bm{g}\| \sim 10^{-5}$. At mild stiffness this residual is tolerable; at $\kappa \gg 1$ it becomes catastrophic (Exp~2).

\begin{table}[t]
\centering
\caption{SMIB: $L_2$ relative error and constraint violation (10 test scenarios).}
\label{tab:smib}
\small
\begin{tabular*}{\columnwidth}{@{\extracolsep{\fill}}lcccc@{}}
\toprule
Method & $\delta$ & $\omega$ & $V$ & $\|\bm{g}\|_\text{avg}$ \\
\midrule
Penalty          & 2.14e-3          & 1.62e-2 & 2.74e-4          & 4.61e-5          \\
FL~\cite{zhang2025fl} & 1.14e-4     & 1.33e-3 & 5.50e-5          & 7.58e-5          \\
\textbf{Newton}  & \textbf{2.91e-5} & 1.57e-3 & \textbf{3.65e-6} & \textbf{2.56e-16} \\
\bottomrule
\end{tabular*}
\end{table}

\subsection{Exp 2: GFM, Standard vs.\ Extended Newton}
\label{sec:gfm_setup}

We compare five methods on the GFM system ($\kappa \approx 4{,}712$) under a voltage sag ($v_g = 0.6$\,pu, 0.5\,s horizon, $T_w = 5$\,ms, 100 recursive windows). This stiffness ratio is representative of converter-dominated grids and is roughly $500\times$ stiffer than the SMIB. No existing soft-constraint PINN method has demonstrated convergence at this level of stiffness.

\begin{table}[t]
\centering
\caption{GFM $L_2$ relative error (\%) and convergence status, voltage sag $v_g\!=\!0.6$\,pu, $\kappa \approx 4{,}712$.}
\label{tab:main_result}
\small
\begin{tabular*}{\columnwidth}{@{\extracolsep{\fill}}lccccc@{}}
\toprule
Method & $n_\text{out}$ & Slow & Fast & Alg & $\|\bm{g}\|$ \\
\midrule
Penalty~\cite{moya2023daepinn}          & 13 & 39.3 & 53.2 & 55.5 & $10^{-3}$ \\
Std.\ Newton (ours)                      & 13 & 57.0 & 29.7 & 69.5 & $10^{-14}$ \\
AL~\cite{son2023alpinn}                  & 13 & \multicolumn{3}{c}{\textit{failed ($\|\bm{g}\|\!\approx\!0.8$, 500k ep.)}} & 0.8 \\
FL~\cite{zhang2025fl}                    & 13 & \multicolumn{3}{c}{\textit{diverged (ep.\ 409/200k)}} & -- \\
\textbf{Ext.\ Newton (ours)}            & \textbf{7}  & \textbf{1.42} & \textbf{3.73} & \textbf{5.66} & $\bm{10^{-16}}$ \\
\bottomrule
\end{tabular*}
\end{table}

Table~\ref{tab:main_result} makes the central point: at $\kappa \approx 4{,}712$, every existing method either fails to converge or produces errors exceeding 39\%. AL never drove $\|\bm{g}\|$ below 0.8 after 500k epochs; FL diverged entirely. Only extended Newton achieves single-digit errors (1.42\% slow) while maintaining $\|\bm{g}\| \leq 10^{-16}$.

The failure pattern is instructive. Penalty leaves $\|\bm{g}\| \sim 10^{-3}$, amplified to fast-state errors $\sim\kappa \cdot 10^{-3} \approx 5$. Standard Newton eliminates this but forces the network to learn $O(1/\kappa)$-timescale oscillations (57\% error), as IFT gradients divert optimizer capacity to unlearnable fast states. Extended Newton avoids both: constraints and fast dynamics are solved exactly, leaving only smooth slow dynamics for the network.

\begin{figure*}[t]
    \centering
    \includegraphics[width=\textwidth]{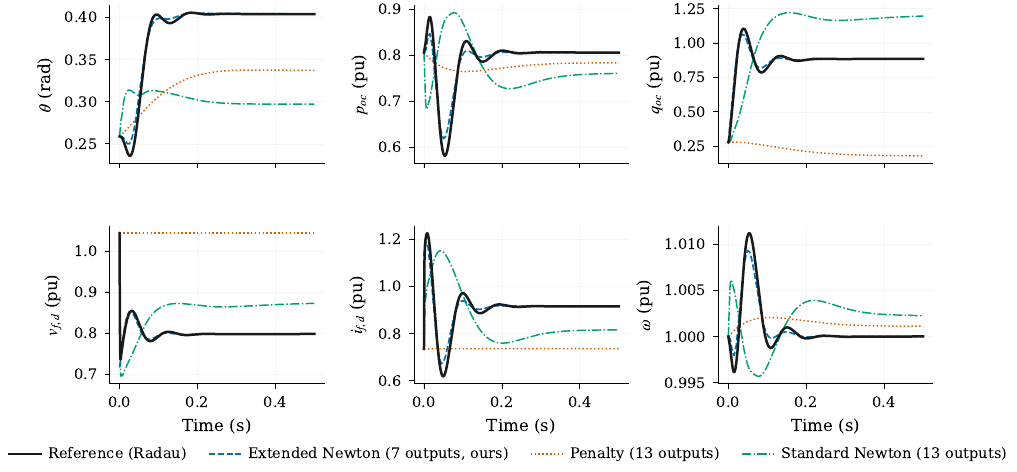}
    \caption{GFM voltage sag ($v_g = 0.6$\,pu): trajectory comparison across 6 states. Extended Newton (dashed) tracks the reference; Penalty (dotted) deviates on $q_\text{oc}$ and $\omega$ due to $\kappa$-amplified $\|\bm{g}\|$; Standard Newton (dash-dot) fails on fast dynamics despite exact $\bm{g}=\bm{0}$. AL failed to converge and FL diverged (omitted).}
    \label{fig:gfm_comparison}
\end{figure*}

\subsection{Exp 3: Two-GFM Cascaded Scalability}
\label{sec:cascaded_exp}

Two GFM inverters ($P_{\text{ref}}=0.806, 0.5$\,pu) connect to the same point of common coupling (PCC) ($r_g=0.01$, $l_g=0.2$\,pu) coupled to an infinite bus ($r_\text{line}=0.02$\,pu), forming a system with $n_x=26$ and $n_z=18$. Each PI-DeepONet is deployed without retraining; coupling is enforced by the cascaded Newton coupler.

\textbf{Cascaded inference.} Each PI-DeepONet predicts its 7 slow states; Algorithm~\ref{alg:cascaded} solves two local $14\times 14$ extended Newton systems in parallel, plus a $2\times 2$ outer Newton for PCC coupling. The outer loop converges in 2--3 steps ($\rho < 1$, Theorem~\ref{thm:convergence}).

\textbf{Voltage sag test.} The infinite-bus voltage drops to 0.8\,pu at $t=0$. Over $T=2$\,s, the cascaded surrogate achieves $0.72\%$ (Inv~1) and $1.16\%$ (Inv~2) slow-state error with $\|\bm{g}\|=0$ at every time point (Table~\ref{tab:cascaded}).

\begin{table}[t]
\centering
\caption{Cascaded $L_2$ error (\%), $v_\infty\!=\!0.8$\,pu, $T\!=\!2$\,s. No retraining; $\|\bm{g}\|=0$.}
\label{tab:cascaded}
\small
\begin{tabular*}{\columnwidth}{@{\extracolsep{\fill}}lcc@{}}
\toprule
State & Inv\,1 (\%) & Inv\,2 (\%) \\
\midrule
$\theta$            & 0.98 & 2.96 \\
$p_\text{oc}$       & 0.64 & 1.14 \\
$q_\text{oc}$       & 0.99 & 0.97 \\
$\xi_d$ / $\xi_q$   & 0.78 / 0.68 & 1.24 / 0.78 \\
$\gamma_d$ / $\gamma_q$ & 0.02 / 1.00 & 0.03 / 1.02 \\
\midrule
\textbf{Slow}       & \textbf{0.72} & \textbf{1.16} \\
Fast / Alg          & 1.11 / 2.03 & 2.19 / 3.57 \\
\bottomrule
\end{tabular*}
\end{table}

Fig.~\ref{fig:cascaded_traj} shows predicted vs.\ reference trajectories over the 2\,s horizon. Inv~2 shows larger errors due to higher reactive power sensitivity at $P_\text{ref}=0.5$.

\begin{figure}[t]
    \centering
    \includegraphics[width=\columnwidth]{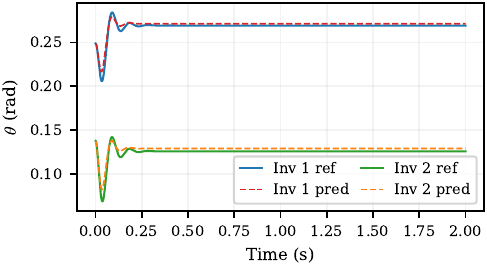}
    \caption{Two-inverter cascaded prediction under voltage sag ($v_\infty=0.8$\,pu, $T=2$\,s). Both PI-DeepONets are deployed without retraining. $\|\bm{g}\|=0$ at all time points.}
    \label{fig:cascaded_traj}
\end{figure}

The key result is zero-shot composition: independently trained models compose without fine-tuning. The 2--3 outer iterations are consistent with $\rho \approx l_f/l_g = 0.08/0.2 = 0.4$ (Theorem~\ref{thm:convergence}).

\subsection{Exp 4: Conformal Prediction as Diagnostic}
\label{sec:cp_exp}

The extended Newton architecture makes two testable predictions: (i)~calibrating slow-state coverage should automatically calibrate fast-state coverage (since fast states are computed, not predicted); and (ii)~out-of-distribution (OOD) inputs should cause simultaneous coverage breakdown across all states, not selective failure. We test both via split conformal prediction (CP)~\cite{shafer2008tutorial}, calibrated on 100 held-out trajectories ($v_g \in [0.5, 1.05]$\,pu) and tested in-distribution (50 trajectories) and OOD ($v_g \in [0.2, 0.5]$\,pu, 50 trajectories), targeting 90\% coverage.

\begin{table}[t]
\centering
\caption{CP coverage (\%, target 90\%). In-dist: $v_g \in [0.5, 1.05]$; OOD: $v_g \in [0.2, 0.5]$.}
\label{tab:cp}
\small
\begin{tabular*}{\columnwidth}{@{\extracolsep{\fill}}lcc@{}}
\toprule
State & In-dist (\%) & OOD (\%) \\
\midrule
$\theta$            & 89.4  & 5.8  \\
$p_\text{oc}$       & 90.0  & 24.8 \\
$q_\text{oc}$       & 91.4  & 5.0  \\
$\xi_d$ / $\xi_q$   & 89.1 / 90.1 & 11.9 / 8.4 \\
$\gamma_d$ / $\gamma_q$ & 91.3 / 90.2 & 11.4 / 0.7 \\
\midrule
\textbf{Average}    & \textbf{90.2} & \textbf{9.7} \\
\bottomrule
\end{tabular*}
\end{table}

Table~\ref{tab:cp} reports per-state coverage. In-distribution, all states achieve 89--91\% (target: 90\%), confirming prediction~(i). OOD, coverage drops to 9.7\% across all states simultaneously, confirming prediction~(ii) and enabling automatic extrapolation detection (Fig.~\ref{fig:cp_bands}). Fast-state errors are consistently $3$--$4\times$ larger than slow-state errors, quantitatively consistent with $C_2 \approx 248$ (Proposition~\ref{prop:qss_error}).

\textbf{Handling non-exchangeability.} Windowed rollouts create temporal dependence, violating standard CP assumptions. We restore exchangeability by treating each trajectory as one data point with score $R_i = \max_t |\hat{x}_s(t) - x_s^\text{true}(t)|$, yielding conservative bands.

\begin{figure*}[t]
    \centering
    \includegraphics[width=\textwidth]{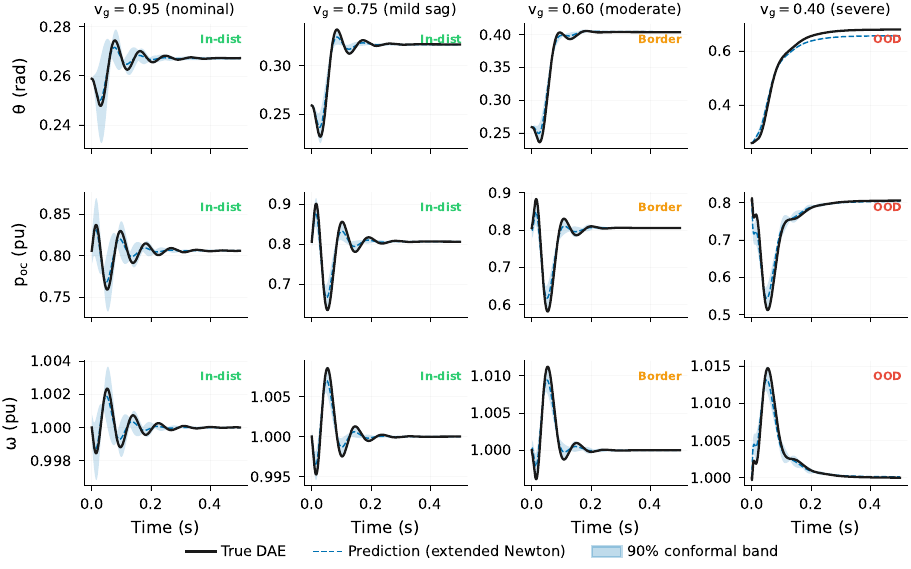}
    \caption{Conformal prediction bands (90\% coverage) across four grid-voltage scenarios. Bands widen from nominal (left) to severe sag (right, OOD), demonstrating automatic extrapolation detection.}
    \label{fig:cp_bands}
\end{figure*}

We derive fast-state CP bands by scaling slow-state bands by an empirical factor ($\approx 5.6$), without independent calibration. The induced coverage of 98.6\% (Table~\ref{tab:cp_induced}) exceeds the directly calibrated 90.6\%, confirming that fast-state uncertainty is fully determined by the Newton solution.

\begin{table}[t]
\centering
\caption{Induced coverage test (in-distribution, target 90\%). ``Direct'': each state group calibrated independently. ``Induced'': fast/algebraic bands derived from slow-state bands scaled by empirical amplification ($\times 5.6$).}
\label{tab:cp_induced}
\small
\begin{tabular*}{\columnwidth}{@{\extracolsep{\fill}}lcc@{}}
\toprule
State group & Direct (\%) & Induced (\%) \\
\midrule
Slow (network) & 90.2 & 90.2 \\
Fast (Newton) & 90.6 & 98.6 \\
Algebraic (Newton) & 90.7 & 93.9 \\
\bottomrule
\end{tabular*}
\end{table}

\subsection{Exp 5: Ablation Studies}
\label{sec:ablation}

Table~\ref{tab:ablation} reports sensitivity to window length and partition specification. Accuracy is stable across $T_w = 1$--$10$\,ms (1.34--1.47\% slow error), but degrades by $3.1\times$ at $T_w = 20$\,ms as more slow-state drift accumulates per window. For partition sensitivity, moving two fast states ($i_\text{cv,d/q}$) into the network's prediction set increases error by $6.9\times$, confirming that the network struggles to learn fast oscillations.

\begin{table}[t]
\centering
\caption{Ablation studies on the voltage sag scenario ($v_g = 0.6$\,pu, $T = 0.5$\,s). Top: window length (standard partition). Bottom: partition (misspecified, $T_w = 5$\,ms).}
\label{tab:ablation}
\small
\begin{tabular*}{\columnwidth}{@{\extracolsep{\fill}}lccc@{}}
\toprule
& Slow (\%) & Fast (\%) & Alg (\%) \\
\midrule
\multicolumn{4}{@{}l}{\textit{Window length sensitivity}} \\
\quad $T_w = 1$\,ms ($T_w/\tau_\text{f} \approx 2$) & 1.39 & 3.73 & 5.67 \\
\quad $T_w = 2$\,ms ($T_w/\tau_\text{f} \approx 4$) & 1.34 & 3.69 & 5.62 \\
\quad $T_w = 5$\,ms ($T_w/\tau_\text{f} \approx 10$) & 1.42 & 3.73 & 5.66 \\
\quad $T_w = 10$\,ms ($T_w/\tau_\text{f} \approx 20$) & 1.47 & 3.72 & 5.59 \\
\quad $T_w = 20$\,ms ($T_w/\tau_\text{f} \approx 40$) & 4.47 & 6.34 & 9.22 \\
\midrule
\multicolumn{4}{@{}l}{\textit{Partition sensitivity ($T_w = 5$\,ms)}} \\
\quad Standard ($n_\text{s}\!=\!7,\, n_\text{f}\!=\!6$) & 1.42 & 3.73 & 5.66 \\
\quad Fast$\to$slow ($n_\text{s}\!=\!9,\, n_\text{f}\!=\!4$) & 9.73 & 13.60 & 25.96 \\
\bottomrule
\end{tabular*}
\end{table}

% ══════════════════════════════════════════
\subsection{Exp 6: Cross-Domain Validation on Robertson Stiff DAE}
\label{sec:robertson}

We apply the framework to System C (Robertson kinetics, defined in Section~\ref{sec:test_systems}):
\begin{equation}
\begin{aligned}
\dot y_1 &= -k_1 y_1 + k_2 y_2 y_3, \\
\dot y_2 &= \phantom{-}k_1 y_1 - k_2 y_2 y_3 - k_3 y_2^2, \\
0 &= y_1 + y_2 + y_3 - 1,
\end{aligned}
\label{eq:robertson}
\end{equation}
with $k_1 = 0.04$, $k_2 = 10^4$, $k_3 = 3\times 10^7$, $\bm{y}(0) = (1, 0, 0)$. Following~\eqref{eq:decomp}, $\bm{x}_\text{s} = (y_1)$, $\bm{x}_\text{f} = (y_2)$, $\bm{z} = (y_3)$. We start at $T_0 = 10^{-2}$\,s with the IC from one Radau evaluation over $[0, T_0]$, after the initial QSS boundary layer (Limitation~(b)).

\textbf{ODE vs.\ DAE formulation.} To isolate the contribution of algebraic constraint enforcement from QSS reduction, we compare two variants of the same architecture and training schedule: the \emph{DAE} formulation predicts $y_1$ alone and solves the combined 2D system $[f_\text{fast};\, g] = 0$ for $(y_2, y_3)$; the \emph{ODE} baseline predicts $(y_1, y_3)$ and solves only $f_\text{fast} = 0$ for $y_2$, treating conservation as an ODE property. A closed-form quadratic warm-start ($y_3 = 1 - y_1 - y_2$ substituted into $f_\text{fast}=0$) yields Newton convergence in 1--2 iterations to $\|F\|_\infty \le 10^{-16}$.

Table~\ref{tab:robertson} and Fig.~\ref{fig:robertson} report the comparison. Slow-state errors are comparable across formulations (both reach the floating-point precision floor of 14-window recursive rollout), but the DAE formulation enforces $\|g\| = 0$ \emph{by arithmetic identity}: $y_3 = 1 - y_1 - y_2$ is one floating-point operation, so the residual is exact, not merely $10^{-14}$. Without this structure, the ODE formulation drifts to $\sim\!10^{-4}$. The DAE variant additionally uses 18\% fewer parameters and trains 30\% faster on identical hardware, with no accuracy penalty. This confirms that the framework generalizes unchanged to a DAE outside the power-system domain. Delegating any boundary-solvable equation to the implicit solver (the design principle of Section~\ref{sec:discussion}) strictly dominates whenever an algebraic invariant is available.

\begin{table}[t]
\centering
\caption{Robertson stiff DAE ($\kappa$ ranges $10^2$--$10^5$, 7-decade time-scale spread, 14 log-time windows): ODE vs.\ DAE formulation. Identical architecture, training schedule, and IC sampling.}
\label{tab:robertson}
\small
\begin{tabular*}{\columnwidth}{@{\extracolsep{\fill}}lcc@{}}
\toprule
& ODE form. & DAE form. \\
\midrule
Network outputs & 2 ($y_1, y_3$) & \textbf{1} ($y_1$) \\
Network parameters & 118{,}852 & 97{,}826 \\
Newton system & 1D ($f_\text{fast}=0$) & 2D ($[f_\text{fast};\, g]=0$) \\
Newton residual $\|F\|_\infty$ & $2.1\times 10^{-16}$ & $9.0\times 10^{-17}$ \\
\midrule
$y_1$ rel.\ $L_2$ error  & 0.0036\% & 0.0033\% \\
$y_2$ rel.\ $L_2$ error  & 0.0059\% & 0.0126\% \\
$y_3$ rel.\ $L_2$ error  & 0.0089\% & 0.0044\% \\
\midrule
Conservation $\max\|g\|$ & $1.1 \times 10^{-4}$ & \textbf{0} \\
Best training loss & $4.7\times 10^{-9}$ & $3.1\times 10^{-9}$ \\
Training time (T4 GPU) & 2{,}107\,s & 1{,}467\,s \\
\bottomrule
\end{tabular*}
\end{table}

\begin{figure}[t]
\centering
\includegraphics[width=\columnwidth]{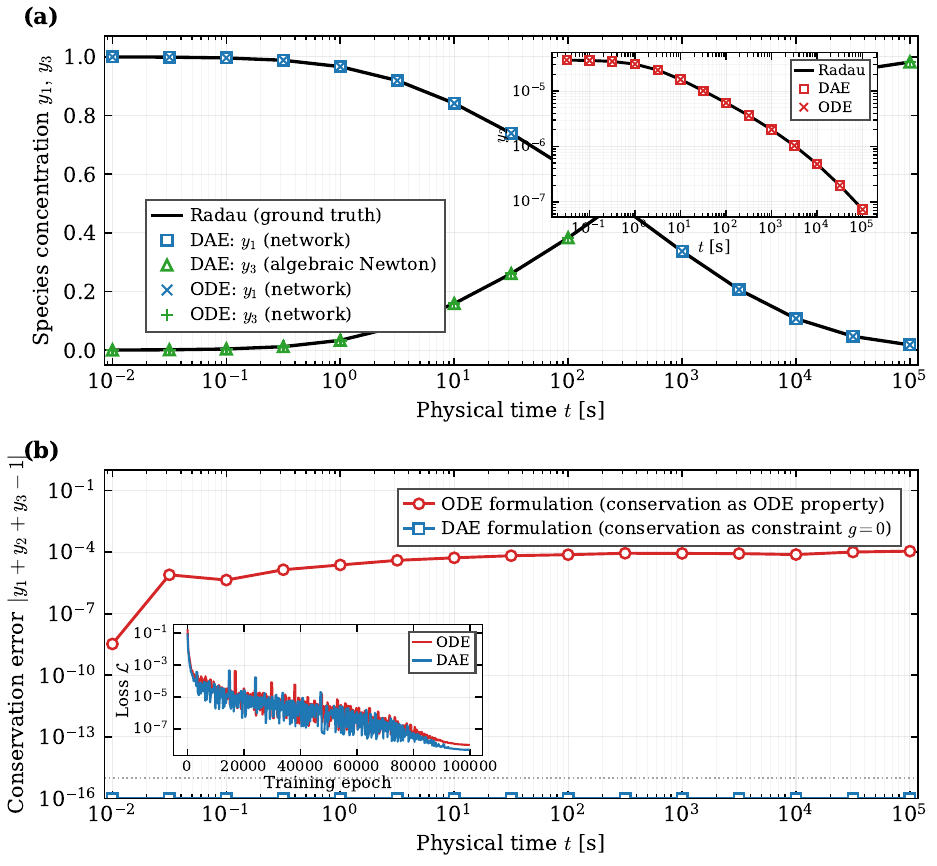}
\caption{Robertson stiff DAE: ODE vs.\ DAE formulation. (a)~Both formulations match the Radau ground truth across seven decades of time. (b)~Conservation error: the DAE formulation rests at machine zero (plotted at $10^{-16}$ for visibility), while the ODE formulation drifts at $\sim\!10^{-4}$. Insets: $y_2$ on log--log scale; training-loss curves under identical hyperparameters.}
\label{fig:robertson}
\end{figure}

% ══════════════════════════════════════════
\section{Discussion}
\label{sec:discussion}

\textbf{Design principle.} Any physical equation solvable at the window boundary (algebraic constraints, QSS conditions, thermodynamic equilibria) should be delegated to the implicit solver, not the network. This principle extends to chemical reaction networks, multi-body mechanics, and building energy systems (fast heating, ventilation, and air conditioning (HVAC) transients vs.\ slow thermal mass).

\textbf{Role of IFT gradients.} The $L_\kappa$-scaled coupling in~\eqref{eq:total_deriv} dominates the training gradient on stiff systems and is structurally absent from penalty methods; extended Newton additionally removes the need to predict fast states, jointly accounting for the $28\times$--$40\times$ improvement on the GFM.

\textbf{QSS validity and robustness.} The method requires $T_w \gg \tau_\text{fast}$ and stable fast eigenvalues. For the GFM, $T_w/\tau_\text{fast} \approx 10$; accuracy is robust across $T_w = 1$--$10$\,ms but degrades at 20\,ms (Table~\ref{tab:ablation}). Zhao et al.~\cite{zhao2025gfm} independently validated this timescale separation.

\textbf{Limitations.} (a)~The partition must be known a priori (eigenvalue clustering can automate this; misspecification costs $6.9\times$, Table~\ref{tab:ablation}). (b)~QSS error is nonzero in the first window after a disturbance, and during initial boundary layers (e.g., $t < 10^{-3}$\,s for Robertson, where $y_2(0)=0$ violates QSS). (c)~Cascaded convergence is validated for $N=2$; larger systems need sparse network Newton. (d)~CP provides marginal, not conditional, coverage.

\textbf{Practical considerations.} Analytical Jacobians avoid automatic differentiation (AD) overhead; AD substitutes at higher cost for black-box models. GPU efficiency requires batched solves and mixed-precision arithmetic; our implementation handles $n_\text{f}+n_z \leq 20$ with negligible overhead. Training cost is 4.7\,h on a single A100 GPU; a 0.5\,s rollout takes $\sim$0.1\,s at inference vs.\ 2--5\,s for Radau~IIA. The Newton-in-the-loop interface supports compositional deployment~\cite{anderson_fouad}, with training cost scaling by the number of component classes, not system size.

% ══════════════════════════════════════════
\section{Conclusion}
% ══════════════════════════════════════════

This paper presented an extended Newton implicit layer that unifies algebraic constraint enforcement and QSS reduction in a single differentiable solve, enabling simulation-free operator learning for stiff DAEs. The network learns only what physics cannot determine exactly: the slow dynamics. On a grid-forming inverter ($\kappa \approx 4{,}712$), this reduces error by $28\times$--$40\times$ over existing methods with $\|\bm{g}\|=0$ guaranteed. Cascaded implicit layers compose independently trained component models without retraining. Cross-domain validation on the Robertson stiff DAE ($\kappa$ up to $10^5$) confirms generalization beyond power systems, with the DAE formulation strictly dominating an ODE formulation in both constraint satisfaction and computational efficiency whenever an algebraic invariant is available.

Several directions remain open. First, the slow/fast partition is currently specified manually; data-driven partitioning via eigenvalue clustering or sensitivity analysis could automate this step. Second, the cascaded scheme assumes weak inter-component coupling ($\rho < 1$); scaling to strongly coupled networks (e.g., bulk power grids with hundreds of inverters) requires sparse-Newton linear algebra at the coupler. Third, the conformal prediction diagnostic provides marginal coverage; adaptive or conditional variants~\cite{gibbs2021adaptive} could tighten the bands.

% ══════════════════════════════════════════
% References
% ══════════════════════════════════════════


\begin{thebibliography}{99}

\bibitem{stiasny2024pinnsim}
J.~Stiasny, B.~Zhang, and S.~Chatzivasileiadis, ``PINNSim: A simulator for power system dynamics based on physics-informed neural networks,'' \textit{Electric Power Systems Research}, vol.~235, 2024.

\bibitem{brenan1996dae}
K.~E.~Brenan, S.~L.~Campbell, and L.~R.~Petzold, \textit{Numerical Solution of Initial-Value Problems in Differential-Algebraic Equations}, Classics in Applied Mathematics, vol.~14. SIAM, 1996.

\bibitem{hairer1996ode2}
E.~Hairer and G.~Wanner, \textit{Solving Ordinary Differential Equations~II: Stiff and Differential-Algebraic Problems}, 2nd ed. Springer, 1996.

\bibitem{robertson1966}
H.~H.~Robertson, ``The solution of a set of reaction rate equations,'' in \textit{Numerical Analysis: An Introduction}, J.~Walsh, Ed. London, U.K.: Academic Press, 1966, pp.~178--182.

\bibitem{moya2023daepinn}
C.~Moya and G.~Lin, ``DAE-PINN: A physics-informed neural network model for simulating differential-algebraic equations with application to power networks,'' \textit{Neural Computing and Applications}, vol.~35, pp.~3789--3804, 2023.

\bibitem{son2023alpinn}
H.~Son, S.~W.~Cho, and H.~J.~Hwang, ``Enhanced physics-informed neural networks with augmented Lagrangian relaxation method (AL-PINNs),'' \textit{Neurocomputing}, vol.~548, p.~126424, 2023.

\bibitem{donti2021dc3}
P.~L.~Donti, D.~Rolnick, and J.~Z.~Kolter, ``DC3: A learning method for optimization with hard constraints,'' in \textit{Proc. ICLR}, 2021.

\bibitem{pal2025pnodes}
A.~Pal, A.~Edelman, and C.~Rackauckas, ``Semi-explicit neural DAEs: Learning long-horizon dynamical systems with algebraic constraints,'' \textit{arXiv:2505.20515}, 2025.

\bibitem{golder2025daehardnet}
R.~Golder, B.~N.~Roy, and M.~M.~F.~Hasan, ``DAE-HardNet: A physics constrained neural network enforcing differential-algebraic hard constraints,'' \textit{arXiv:2512.05881}, 2025.

\bibitem{wang2021pideepont}
S.~Wang, H.~Wang, and P.~Perdikaris, ``Learning the solution operator of parametric partial differential equations with physics-informed DeepONets,'' \textit{Science Advances}, vol.~7, no.~40, 2021.

\bibitem{spotorno2025phrpinn}
E.~N.~Spotorno, J.~Leal~Filho, and A.~A.~Fr\"{o}hlich, ``Hard-constrained neural networks with physics-embedded architecture for residual dynamics learning and invariant enforcement in cyber-physical systems,'' \textit{arXiv:2511.23307}, 2025.

\bibitem{moya2023deeponetgrid}
C.~Moya, S.~Zhang, G.~Lin, and M.~Yue, ``DeepONet-Grid-UQ: A trustworthy deep operator framework for predicting the power grid's post-fault trajectories,'' \textit{Neurocomputing}, vol.~535, pp.~166--182, 2023.

\bibitem{koch2024neural}
J.~Koch, M.~Shapiro, H.~Sharma, D.~Vrabie, and J.~Drgo\v{n}a, ``Learning neural differential algebraic equations via operator splitting,'' \textit{arXiv:2403.12938}, 2024.

\bibitem{zhang2025fl}
R.~Zhang, A.~Raghunathan, J.~Shamma, and N.~Li, ``Constrained optimization from a control perspective via feedback linearization,'' in \textit{Proc. NeurIPS}, 2025.

\bibitem{trSQP2024}
X.~Cheng and S.~Na, ``Physics-informed neural networks with trust-region sequential quadratic programming,'' \textit{arXiv:2409.10777}, 2024.

\bibitem{white2023sndes}
A.~White, N.~Kilbertus, M.~Gelbrecht, and N.~Boers, ``Stabilized neural differential equations for learning dynamics with explicit constraints,'' in \textit{Proc. NeurIPS}, 2023.

\bibitem{lueg2025simultaneous}
L.~R.~Lueg, V.~Alves, D.~Schicksnus, J.~R.~Kitchin, C.~D.~Laird, and L.~T.~Biegler, ``A simultaneous approach for training neural differential-algebraic systems of equations,'' \textit{arXiv:2504.04665}, 2025.

\bibitem{lu2021deeponet}
L.~Lu, P.~Jin, G.~Pang, Z.~Zhang, and G.~E.~Karniadakis, ``Learning nonlinear operators via DeepONet based on the universal approximation theorem of operators,'' \textit{Nature Machine Intelligence}, vol.~3, pp.~218--229, 2021.

\bibitem{karampinis2025pioperator}
I.~Karampinis, P.~Ellinas, J.~Vorwerk, and S.~Chatzivasileiadis, ``Neural operators for power systems: A physics-informed framework for modeling power system components,'' \textit{arXiv:2511.05216}, 2025.

\bibitem{tikhonov1952}
A.~N.~Tikhonov, ``Systems of differential equations containing small parameters in the derivatives,'' \textit{Matematicheskii Sbornik}, vol.~31, no.~3, pp.~575--586, 1952.

\bibitem{kokotovic1999}
P.~V.~Kokotovi\'{c}, H.~K.~Khalil, and J.~O'Reilly, \textit{Singular Perturbation Methods in Control: Analysis and Design}. SIAM, 1999.

\bibitem{gear2006}
C.~W.~Gear, ``Towards explicit methods for differential algebraic equations,'' \textit{BIT Numerical Mathematics}, vol.~46, pp.~505--514, 2006.

\bibitem{zhao2025gfm}
X.~Zhao, M.~Netto, and J.~Zhao, ``A novel discrete-time state-space model for decentralized dynamic state estimation of grid-forming inverters,'' \textit{IEEE Trans. Power Syst.}, 2025.

\bibitem{ji2021stiffpinn}
W.~Ji, W.~Qiu, Z.~Shi, S.~Pan, and S.~Deng, ``Stiff-PINN: Physics-informed neural network for stiff chemical kinetics,'' \textit{J.\ Phys.\ Chem.\ A}, vol.~125, no.~36, pp.~8098--8106, 2021.

\bibitem{lee2025fsnn}
N.~Lee and R.~Temam, ``Fast-slow neural networks for learning singularly perturbed dynamical systems,'' \textit{J.\ Comput.\ Phys.}, 2025.

\bibitem{caldana2025neuralode}
M.~Caldana, P.~Mossier, and L.~Pareschi, ``Neural ordinary differential equations for model order reduction of stiff systems,'' \textit{Int.\ J.\ Numer.\ Methods Eng.}, vol.~126, no.~13, e70060, 2025.

\bibitem{angelopoulos2023conformal}
A.~N.~Angelopoulos and S.~Bates, ``Conformal prediction: A gentle introduction,'' \textit{Foundations and Trends in Machine Learning}, vol.~16, no.~4, pp.~494--591, 2023.

\bibitem{shafer2008tutorial}
G.~Shafer and V.~Vovk, ``A tutorial on conformal prediction,'' \textit{Journal of Machine Learning Research}, vol.~9, no.~3, 2008.

\bibitem{romano2019conformalized}
Y.~Romano, E.~Patterson, and E.~Cand\`es, ``Conformalized quantile regression,'' \textit{Advances in Neural Information Processing Systems}, vol.~32, 2019.

\bibitem{gibbs2021adaptive}
I.~Gibbs and E.~Cand\`es, ``Adaptive conformal inference under distribution shift,'' \textit{Advances in Neural Information Processing Systems}, vol.~34, pp.~1660--1672, 2021.

\bibitem{stankeviciute2021conformal}
K.~Stankeviciute, A.~M.~Alaa, and M.~van~der~Schaar, ``Conformal time-series forecasting,'' \textit{Advances in Neural Information Processing Systems}, vol.~34, pp.~6216--6228, 2021.

\bibitem{anderson_fouad}
P.~M.~Anderson and A.~A.~Fouad, \textit{Power Systems Control and Stability}, 2nd ed. Wiley-IEEE Press, 2003.

\end{thebibliography}
\end{document}